\documentclass[runningheads]{llncs}
\usepackage[T1]{fontenc}

\usepackage{graphicx}
\usepackage{hyperref}
\usepackage{amsmath}
\usepackage{amssymb}
\usepackage{mathtools}
\usepackage{algorithm}
\usepackage{algorithmic}
\usepackage{microtype}
\usepackage{graphicx}
\usepackage{subcaption}
\usepackage{booktabs} 
\usepackage{arydshln}
\usepackage{enumitem}
\usepackage{subcaption} 
\usepackage{threeparttable} 
\usepackage{multirow}
\usepackage{tikz}

\usepackage{color}

\newcommand{\ours}{diffGHOST }

\begin{document}
\title{diffGHOST: Diffusion based Generative Hedged Oblivious Synthetic Trajectories}
\titlerunning{diffGHOST}
\author{Florent Guépin\inst{1}\textsuperscript{*}\orcidID{0009-0008-5098-0963} \and
Cheick Tidiani Cisse\inst{2}\textsuperscript{*}\orcidID{0009-0001-9984-8065} \and
Denis Renaud\inst{1}\textsuperscript{*}\orcidID{0000-0003-0381-4862} \and
François Bidet\inst{2}\orcidID{0000-0001-9813-0388} \and
Arnaud Legendre\inst{1}\orcidID{0009-0001-9512-2355}
}
\authorrunning{Guépin et al.}
\institute{Orange Research, Belfort, France \\
\email{\{florent.guepin, denisxavier.renaud, arnaud.legendre\}@orange.com} \and Orange Business, Belfort, France \\
\email{\{cheicktidiani.cisse, francois.bidet\}@orange.com}}
\maketitle
\def\thefootnote{*}\footnotetext[1]{These authors contributed equally to this work}

\begin{abstract}
Trajectories are nowadays valuable information for a wide range of applications. However they are also inherently sensitive, as they contain highly personal information about individuals. Facing this challenge, synthesizing mobility trajectories has emerged as a promising solution to leverage mobility information while preserving privacy. State-of-the-art models, often rely on the false assumptions of generative models implicit privacy and fails to provide privacy guarantees while preserving trajectories utility. Here, we introduce diffGHOST, a conditional diffusion model based on latent space segmentation, designed to answer this challenge. Thus, this paper propose a methodology that identify and mitigate memorization of critical samples using condition segments of a learn latent space.

\end{abstract}
\section{Introduction}
Trajectory information about individuals is a valuable source of information, for a large variety of applications, ranging from research such as epidemiology studies~\cite{tizzoni2014use}, to urban planning~\cite{mahrez2021smart} or environmental studies~\cite{ma2020examining}. However, those trajectory information contain intrinsically private information, such as religious practices or political affiliation information~\cite{georgiadou2019location}. Thus, regulations such as the General Data Protection Regulation (GDPR) in Europe~\cite{eu-gdpr} have arise, and apply to human trajectories datasets, preventing stakeholders from valuable utilization of the underlying private information. Regulations introduced a notion of adversarial effort to produce information about a set of individuals, introducing a fundamental trade-off when publication such datasets: Between the quality of the valuable information referred to as the \textit{utility} and the personal data protection guarantees, referred to as the \textit{privacy}.

Following this rising interest, mathematical guarantees have quickly emerged to provide insurance of privacy. The formalism that quantifies the sufficient level of noise to add in order to achieve a certain tradeoff, is called Differential Privacy (DP), and has been first introduced by Dwork et al.~\cite{dwork2006calibrating}. However, in practice, this doesn't answer the fundamental trade-off privacy-utility as the level of noise one has to use to ensure privacy is too high to achieve satisfaction, especially for trajectories. 
A prolific literature exist to tighten the required magnitude of the noise~\cite{houssiau2022tapas}, with relaxed DP~\cite{dwork2006our}, but the micro-information, however crucial for some usages is destroyed~\cite{soria2017individual}.

To answer this issue, production of synthetic mobility data has emerged as a promising solution to preserve micro-information while providing access to anonymized information~\cite{zhu2023difftraj}. This is because the synthetic trajectories would be synthesized by a model rather than pertain to individuals, and therefore would not be attributable to anyone specifically. In practice however, synthetic data has also been proven vulnerable to privacy attacks~\cite{guepin2023synthetic,meeus2023achilles}. Even state of the art model, such as DiffTraj, based on diffusion model and introduced by~\cite{zhu2023difftraj}, falls vulnerable to privacy attacks as it relies solely on the generation by diffusion model~\cite{matsumoto2023membership}. 

\textit{Contributions.} In this paper, we introduce a novel architecture called \ours for Diffusion based Generative Hedged Oblivious Synthetic Trajectories. This model aims to produce synthetic datasets that provides a strong non-memorization guarantee. To do so, we introduce what is, to the best of our knowledge, the first conditional diffusion model based on latent space of VAE for trajectory generation. To build a strong privacy guarantee, we used a novel idea of segmentation of the distribution conditional to the latent space of a VAE, to later produce strong non-memorization guarantees at generation time.
\label{sec:introduction}

\section{Related Works}
\textbf{Trajectory generation.} Trajectory generation methods can be broadly divided into generative and non-generative approaches. Non-generative methods, such as perturbing real trajectories with random or Gaussian noise~\cite{zandbergen2014ensuring} or combining them~\cite{nergiz2008towards}, are the most straightforward. While simple and sometimes privacy-preserving by construction, these approaches tend to distort the original distribution, producing unrealistic traces with limited utility for real applications. Non-generative solutions thus often fail to address the utility-privacy trade-off, as privacy can only be improved at the cost of utility.\\
This limitation motivates the preference for generative approaches in the literature. These methods leverage deep neural networks to learn the underlying spatio-temporal distribution of real trajectories, from which new synthetic traces can be sampled. Liu et al.~\cite{liu2018trajgans} first introduced the idea of trajectory generation based on Generative Adversarial Networks (GANs)~\cite{goodfellow2020generative}. Since then, a wide range of architectures have been explored: GAN-based~\cite{rao2020lstm}, VAE-based~\cite{chen2021trajvae}, diffusion-based~\cite{zhu2023difftraj}, and transformer-based~\cite{hsu2024trajgpt} solutions. Each comes with its own trade-offs. Diffusion models produce high-quality and diverse samples but may struggle to capture long-range dependencies. Transformer-based models address this through the attention mechanism, yet may lack diversity and typically require map-grid tokenization. Notably, \cite{zhu2023difftraj} proposed conditioning generation on explicit features such as start/end points, duration, or average speed, enabling fine-grained control. In this work, we instead condition our model on discrete latent representations, a compact conditioning that implicitly captures multiple features and groups trajectories according to their spatio-temporal structure.

\textbf{Privacy Guarantees in trajectory generation.}
Discussion over the privacy guarantees in trajectory generation often falls into two distinct categories. The first one grounds the privacy analysis into the inherent property of model generation, using argument based on the anonymity \textit{by design} offered by generative solutions, as described in multiple survey~\cite{miranda2023sok,buchholz2024sok}. Following the argument made by(\cite{cherigui2026dualperspectivesynthetictrajectory} and~\cite{abbar2025trajdd}), we argue these privacy properties need to be formally evaluated and quantified. The other category grounds the privacy analysis on the formal guarantees provided by the addition of differential privacy into their models~\cite{zhang2023lgan}. First, Buchholz et al.~\cite{buchholz2024sok} unveiled that these formal guarantees suffers from problems in their proofs, leading to a false sense of privacy. Second, these solutions come at a great cost in utility as they add noise to the global distribution. By the capacity of our model to segment the original distribution, we unlock the capacity to provide, when needed, local K-anonymity, to tailor the noise addition to each segment's memorization's sensitivity and save the overall utility of the model.

K-anonymity in tabular datasets has been introduced first by Sweeney et al.~\cite{sweeney2002k}. K-anonymity ensures that for every  selection of attributes, at least K persons are concerned by this selection. This notion has been carefully adapted to trajectory information by Gruteser et al.~\cite{gruteser2003anonymous}. It translates that for each couple of spatiotemporal points, there is at least K individuals can be described by these spatiotemporal selection. In our solution, we locally blur the spatiotemporal information of individuals evaluated memorized by the model, to achieve formal local K-anonymity, removing the memorization threat these segments had. 
\label{sec:relatedwork}

\section{Background}
\subsection{Trajectories}
A trajectory, as the capture of individual's mobility, is typically defined as a chronologically ordered sequence of event. Formally, with $U$ the set of all individual in the considered dataset, we can define a trajectory as follows. An individual $u \in U$ has a trajectory $T_u$ defined as $T_u = [(x_{t_1},y_{t1}),\cdots,(x_{t_n},y_{t_n})]$. Where each $(x_i,y_i)$ defined the geographic position of $u$ in a given reference system (e.g. $(latitude, longitude)$) at time $t_i$; with $t_i$ being explicitly or implicitly defined. The set of all trajectories is called $T$.  Often, trajectory datasets can be enriched with categorical information, such as global information about the owner's trace, or semantic details about visited points (Points of Interest). In such case, the trace $T_u$ can be seen as $T_u = [[(x_{t_1},y_{t_1}), c_{t_1}^1, \cdots,c_{t_1}^m],\cdots,[(x_{t_n},y_{t_n}), c_{t_n}^1, \cdots, c_{t_n}^m]]$ where each $c^i_j$ are one categorical information. Thus, a user's trace is defined as the collection of all their trajectories. In this paper, we will refer to the trajectory dataset as $D$, the dataset containing both all users information $U$ and all the trajectories $T$.

\subsection{VAEs and latent space}
\label{subsec:VAEandlatent}
Variational autoencoders (VAEs)~\cite{kingma2013auto} are types of generative models that define $p_{\theta}(\mathbf{x}\mid \mathbf{z})$ with prior $p(\mathbf{z})$ (often $\mathcal{N}(\mathbf{0}, \mathbf{I})$) and a parametric model (encoder) $q_{\phi}(\mathbf{z}\mid \mathbf{x})$ also called inference model.
The goal is to optimize the variational parameters (the encoder) $\phi$, so that the true posterior $p_{\theta}(\mathbf{x}\mid \mathbf{z})$, mathematically intractable, is equivalent to the approximated posterior $q_{\phi}(\mathbf{z}\mid \mathbf{x})$. Thus, the optimization objective of VAEs is the evidence lower bound (ELBO): $\mathcal{L}_{\text{ELBO}}(\theta,\phi; \mathbf{x}) =\, \mathbb{E}_{q_{\phi}(\mathbf{z}\mid \mathbf{x})} \big[\log p_{\theta}(\mathbf{x}\mid \mathbf{z})\big] - \mathrm{KL}\!\big(q_{\phi}(\mathbf{z}\mid \mathbf{x}) \,\|\, p(\mathbf{z})\big).$
In practice, the encoder has 2 outputs: per-dimension mean $\boldsymbol{\mu}_{\phi}(\mathbf{x})$ and (log)variance $\log \boldsymbol{\sigma}^2_{\phi}(\mathbf{x})$, used during sampling through the reparameterization trick: $\mathbf{z} \;=\; \boldsymbol{\mu}_{\phi}(\mathbf{x}) \;+\; \boldsymbol{\sigma}_{\phi}(\mathbf{x}) \odot \boldsymbol{\epsilon}$ where $\boldsymbol{\epsilon} \sim \mathcal{N}(\mathbf{0}, \mathbf{I})$. 

\subsection{K-anonymity}

First introduced by Samarati et al.~\cite{samarati1998protecting} and significantly amended in Sweeney et al.~\cite{sweeney2002k}, it defined a key notion for data privacy. For tabular dataset, the idea is that for each tabular value $c_i$, $c_i$ must be present at least $k$ time any queries on the database to protect it's privacy.  In the context of trajectory generation, the previous definition has been carefully adapted by Abul et al.~\cite{abul2008never}. The adapted version of the definition introduced a new parameter $\delta$ such that we can define the $(k,\delta)$-anonymity. $(k, \delta)$-anonymity ensure that at any time, there is at least $k$ traces in the spherical volume of radius $\delta$ for any subset $S \subseteq D$ and center $c \in S$.

\subsection{Diffusion models}
Diffusion models are likelihood-based generative models that learn data distribution by defining a noising (forward) process $\mathbf{q(x_t\mid x_{t-1})}$, a Markov chain, that gradually perturbs data into a simple distribution (e.g., Gaussian distribution) and learns a parametric model $p_{\theta} (x_{t-1} \mid x_t)$ that reverses the forward process~\cite{ho2020denoising,nichol2021improved}.

\begin{figure*}
    \centering
    \includegraphics[width=12.0cm]{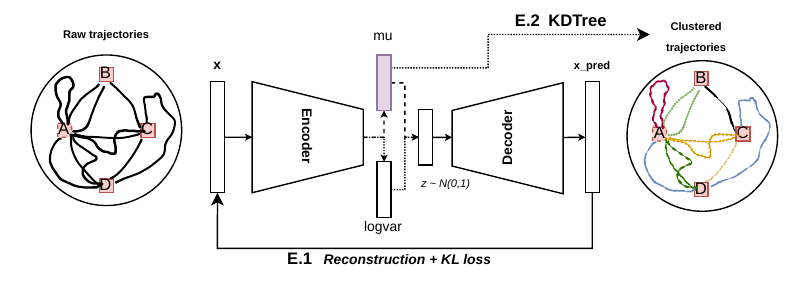}
    \caption{Illustration of trajectory projection via VAE and KDTree}
    \label{fig:bigpicture}
\end{figure*}

\paragraph{\textbf{Forward (or diffusion) process:}} Consists of $T$ steps that add Gaussian noise with a variance schedule $\{\beta_t\}_{t=1}^T$. Let $\alpha_t \!=\! 1-\beta_t$ and $\bar\alpha_t \!=\! \prod_{s=1}^t \alpha_s$. The forward chain factorizes as $q(\mathbf{x}_{1:T}\mid \mathbf{x}_0) = \prod_{t=1}^T q(\mathbf{x}_t \mid \mathbf{x}_{t-1})$, and $q(\mathbf{x}_t \mid \mathbf{x}_{t-1}) = \mathcal{N}\!\big(\sqrt{\alpha_t}\,\mathbf{x}_{t-1},\, \beta_t \mathbf{I}\big)$.
    
An important property of the forward process both theoretical and practical, is that it allows sampling $\mathbf{x_t}$ at any timestep $\mathbf{t}$ from $\mathbf{x_0}$. This close form, also called reparameterization trick can be defined as follows: $q(\mathbf{x}_t \mid \mathbf{x}_0) = \mathcal{N}\!\big(\sqrt{\bar\alpha_t}\,\mathbf{x}_0,\,(1-\bar\alpha_t)\mathbf{I}\big)$, and $\mathbf{x}_t = \sqrt{\bar\alpha_t}\,\mathbf{x}_0 + \sqrt{1-\bar\alpha_t}\,\boldsymbol{\epsilon},\;\; \boldsymbol{\epsilon}\!\sim\!\mathcal{N}(\mathbf{0},\mathbf{I})$. The time-steps $\mathbf{T}$ as well as the noise variance scheduler (e.g. linear, cosine etc) are considered as parameters to tune with respect to the task.
\paragraph{\textbf{Reverse process:}} Aims to recover $\mathbf{x}_0$ by iteratively denoising from  $\mathbf{x}_T \!\sim\! \mathcal{N}(\mathbf{0}, \mathbf{I})$ down to $\mathbf{x}_0$.  
At each step, a neural network predicts a denoising target---either the noise  
$\boldsymbol{\epsilon}$, the clean sample $\mathbf{x}_0$, or the $v$-parameterization---and  
defines a Gaussian transition $
p_{\theta}(\mathbf{x}_{t-1}\mid \mathbf{x}_t)
= \mathcal{N}\!\big(\boldsymbol{\mu}_{\theta}(\mathbf{x}_t, t),\, \sigma_t^2 \mathbf{I}\big)$.

In practice, the mean-squared error is optimized on the chosen target over  
sampled timesteps (e.g uniformly, based on SNR etc). At sampling time, the update is $    \mathbf{x}_{t-1} = \frac{1}{\sqrt{\alpha_t}} \!\left(\mathbf{x}_t - \frac{1-\alpha_t}{\sqrt{1-\bar{\alpha}_t}} \,\boldsymbol{\epsilon}_{\theta}(\mathbf{x}_t,t) \right) + \sigma_t\,\mathbf{z}$, and $\mathbf{z}\!\sim\!\mathcal{N}(\mathbf{0},\mathbf{I})$ (if $t>1$). Setting $\sigma_t = 0$ yields deterministic DDIM sampling \cite{song2020denoising}, while using a nonzero $\sigma_t$ recovers stochastic DDPM sampling.
\label{sec:background}

\section{Model Framework}
Our methodology for synthesizing trajectories with guarantees of non-memorization follows 5 distinct steps, namely E$_1$, $\cdots$, E$_4$. First, E$_1$ presents the training step of our VAE, which consists in learning latent representation of trajectories. Second, E$_2$ creates the clusterisation of the trajectories by using a KDtree algorithm over the trajectories latent representation created in E$_1$. Third, E$_3$ trains a conditional diffusion model, where the trajectories used to train the model are the same trajectories from the training set used in E$_1$ and the condition vectors come from E$_2$ outputs (discret latent representation). Finally, E$_4$ establishes a non-memorization guarantee over the traces synthesized at the end of step E$_3$. 

\begin{figure*}
    \centering
    \includegraphics[width=0.75\linewidth]{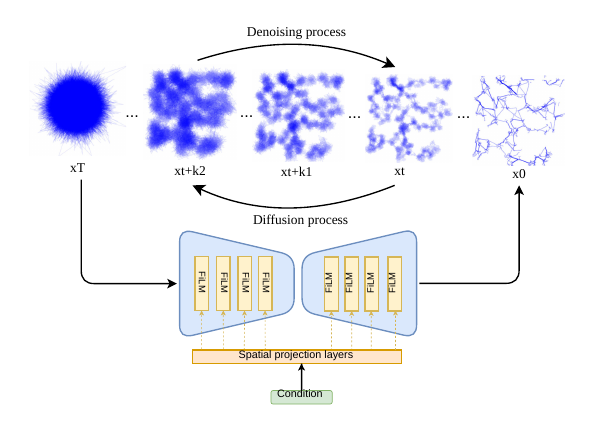}
    \caption{Illustration of trajectory diffusion model architecture E$_3$}
    \label{fig:diffusion_big_picture}
\end{figure*}

\subsection{E$_1$: VAE training}
\label{subsec:vaetraining}

Let $D_{\text{train}}$ and $D_{\text{test}}$ denote the training and testing sets, with $K_{\text{train}}$ and $K_{\text{test}}$ trajectory samples $T_u$, respectively. Each $T_u$ is a trajectory of fixed length $L=256$. The goal, as in the ELBO equation (Sec.~\ref{subsec:VAEandlatent}), is to learn a latent representation that reconstructs $T_u$ accurately while keeping the approximate posterior $q_{\phi}(\mathbf{z}\mid T_u)$ close to a standard normal distribution. To achieve this goal, we use a 1D CNN encoder–decoder architecture that takes in a normalized trajectory ($[-1,1]$) and  outputs $\boldsymbol{\mu}(T_u)$ and $\boldsymbol{\sigma}(T_u)$. The latent $\mathbf{z(T_u)}$ is then sampled from ($\boldsymbol{\mu}(T_u)$, $\boldsymbol{\sigma}(T_u)$) as in the parameter trick equation (Sec.~\ref{subsec:VAEandlatent}) and used by the decoder to reconstruct $\widehat{T}_u$. 

Thus, the final loss of the VAE is a weighted sum of two terms: $\mathcal{L}_{\text{VAE}}(T_u) = \lambda_{\text{rec}}\,\|\,T_u - \widehat{T}_u\,\|_2^2 + \beta\,\mathrm{KL}\!\big(q_{\phi}(\mathbf{z}\mid T_u)\,\|\,\mathcal{N}(\mathbf{0},\mathbf{I})\big)$ where the MSE term measures reconstruction quality (maximizing $p_\theta(x\mid z)$), while the KL term minimizes the divergence between the approximate posterior $q_\phi(z\mid x)$ and the prior $p(z)$. We set $\lambda_{\text{rec}}= 1.0$ and $\beta = 0.1$ empirically.

\subsection{E$_2$: Segmentation of the latent space.} 
\label{subsec:segmentation}
Once the training of the VAE is done, each trajectory $T_u \in D_{train}$ is mapped to its $\boldsymbol{\mu}(T_u)$ using the trained encoder. Thus, the mean of trajectories $\mu(T_u)$ are used as the latent representations. After projecting all the training points into the latent space of dimension $n$, the latent space segmentation is performed using a KDTree. In the rest of the paper, for each trajectory $T_u \in D_{train}$, we will call $C_u$ the hyper-rectangle in which $z_{T_u}$ is contained. 

\subsection{E$_3$: Diffusion training}
\label{subsec:diffusiontraining}

The segmentation of the VAE latent space is followed by the training of our conditional diffusion model to generate synthetic trajectories having similar topological properties given a condition vector $C_u$. Thus, each trajectory $T_u$ is paired with a condition vector $C_u$  that encodes the topological properties of its corresponding segment. Figure \ref{fig:diffusion_big_picture} illustrates our diffusion model architecture.

\textbf{Architecture.} As in many diffusion models, we employ a U-Net-based architecture with self-attention \cite{vaswani2017attention} specifically designed for 1D sequential data. The architecture is composed mainly of CNNs~\cite{krizhevsky2012imagenet} and MLP for embeddings injection. The network consists of four encoder blocks and four decoder blocks, each containing residual connections~\cite{he2016deep}. Skip connections bridge corresponding encoder-decoder layers to preserve fine-grained spatial information across scales. Each residual block comprises two convolutional layers with kernel size 3, followed by normalization~\cite{ioffe2015batch} and activation functions.

\textbf{Temporal and conditional embeddings.}
Diffusion timesteps $t \in [0, T]$ is encoded using sinusoidal positional embeddings~\cite{vaswani2017attention}, which are then projected through MLP layers to produce  embeddings for different encoder/decoder levels. On the other side, the condition vector $C_u$, is embedded and upsampled to match encoder/decoder spatial dimension trough the spatial projection layers as illustrated in figure \ref{fig:diffusion_big_picture}. These embeddings are then injected into the U-Net model via the FiLM (Feature-wise Linear Modulation) \cite{perez2018film} mechanism across multiple levels of the encoder and decoder, giving full freedom to the model on how and where to efficiently use the embeddings while preserving other important features.
Furthermore, we integrated Classifier-Free Guidance (CFG)~\cite{ho2022classifier} to have full control on the conditional generation process. CFG allows to control the diversity of generated samples, giving a  way to operates on privacy concerns. Thus, the noise prediction during the diffusion process is written: $\epsilon_\theta(T_u^t, t, C_u) = \epsilon_\theta(T_u^t, t, C_0) + s \times ( \epsilon_\theta(T_u^t, t, C_0) - \epsilon_\theta(T_u^t, t, C_u))$  where $C_0$ corresponds to the null condition vector (all zero), essentially for unconditional generation. In the case where the guidance scale $s = 0$, the model can be viewed as an unconditional diffusion model. For $s \geq 1$, the model is a conditional diffusion model. The higher the guidance scale, the more the model constraints to the input condition. However, increasing the guidance scale reduces diversity. For more diversity, guidance scale should stick between 0 and 1. For our privacy methodology, as it will be presented in section~\ref{subsection:privacy_framework}, we need to ensure that the model respects the condition input, requirement enforced with CFG.
 
\textbf{Training objective.} Following the DDPM framework~\cite{ho2020denoising}, we train our model to predict the noise $\boldsymbol{\epsilon} \sim \mathcal{N}(\mathbf{0}, \mathbf{I})$ added during the forward diffusion process. At each training step, we sample a random timestep $t \sim \mathcal{U}(1, T)$ and a noise vector $\epsilon$ then compute the noised trajectory $T_u^t = \sqrt{\bar\alpha_t}\,\mathbf{T}_u + \sqrt{1-\bar\alpha_t}\,\boldsymbol{\epsilon}$ where $\bar\alpha_t$ are the cumulative products of the noise schedule values. The denoising network $\boldsymbol{\epsilon}_{\theta}(T_u^t,t,C_u)$ is trained to minimize the simplified variational lower bound: $\mathcal{L}_{loss} = \mathbb{E}_{T_u, \epsilon, t, C_u}\left[\|\epsilon - \epsilon_\theta(T_u^t, t, C_u)\|^2_2\right]$, where the expectation is taken over the training dataset batches, random noise samples, uniformly sampled timesteps, and corresponding condition vectors.
By conditioning on $C_u$ throughout the denoising process  the model learns to generate trajectories that navigate the latent space in accordance with the underlying tessellation topology, ensuring geometric consistency and control over generation.

\subsection{E$_4$: Privacy guarantees}
\label{subsection:privacy_framework}

Our goal is to provide privacy guarantees using the latent space segmentation to condition our diffusion model. To achieve this, we compute, for each segment, the sensitivity of the condition based on the Fréchet distance between each synthetic and real traces mapped or conditioned by the same segment. We first extract the set of real records projected in the given segment $C_u$, more formally, we compute $K_{\mid_{C_u}} = \{ T_u \mid z_{T_u} \in C_u\}$. Following established methodology in trajectory privacy literature~\cite{bonnaire2025diffusion,carlini2023extracting,favero2025bigger}, we assess memorization vulnerability using nearest-neighbor analysis: For each synthetic record $S_{C_u}^i$, we compute the Fréchet distance with every real record $T_u \in K_{\mid_{C_u}}$. We then sort the results, and finally apply a threshold $k$ to the quotient between the distance to the first and the second closest neighbor ($T_u^1$ and $T_u^2$). Formally, we compute: $\frac{||S_{C_u}^i - T_u^1||}{||S_{C_u}^i - T_u^2||} < k$. A synthetic trajectory is called \textit{memorized} if the distance ratio verify the preceding equation, indicating the synthetic sample is disproportionately similar to a single training example. Limitations and interpretations of this metric are discussed in Section~\ref{sec:discussion}. Based on this evaluation, we are able to identify how many synthetic samples $S_{C_u}^i$ from the condition $C_u$ are memorized by the model. Based on this number, we can define if the condition $C_u$ is deemed vulnerable. If this is the case, we apply Laplacian noise to each synthesized trajectory, based on a bound $\delta$. The goal of this noise addition is to ensure an indistinguishability within the synthesized trajectory generated using the condition $C_u$, while preserving the utility of the other neighbors condition. The idea behind our noise addition is to constraint the resulting trajectories to a ball of radius $\delta$ and center $T_c$, such that it does not affect the overall utility, while providing guarantees of $K_{\mid_{C_u}}$-anonymity for this specific condition. More formally, the noise addition is made as follows: Let $S_{C_u}^c$ be the barycenter of the generated trajectories for a given condition $C_u$. We compute the maximum distance between the barycenter and the border of the ball by computing $\Delta S_{max} = |\delta- \sqrt{(S_{C_u}^x-S_{C_u}^{x_c})^2 + (S_{C_u}^y-S_{C_u}^{y_c})^2}|$. Then, we sample a Laplacian noise that almost always keep the new point within the ball of radius $\delta$ by fixing $b=\frac{\Delta S_{max}}{100}$ and $S_{C_u}^{new} = (L(S_{C_u}^x, b),L(S_{C_u}^y, b))$. We then ensure that $S_{C_u}^{new}$ stays within the boundaries of the ball, by computing $\Delta_{new} = \delta- \sqrt{(S_{C_u}^{x_{new}}-S_{C_u}^{x_c})^2 + (S_{C_u}^{y_{new}}-S_{C_u}^{y_c})^2}$ then if $\Delta_{new} < 0$, it means the noise addition get $S_{C_u}^{new}$ out of the ball of radius $\delta$, so we keep the edge of the ball as the new point by fixing $S_{C_u}^{new'} = S_{C_u}^c + \delta*\frac{S_{C_u}^{new} - S_{C_u}^c}{\sqrt{(S_{C_u}^{x_{new}}-S_{C_u}^{x_c})^2 + (S_{C_u}^{y_{new}}-S_{C_u}^{y_c})^2}}$. This noise addition ensures we keep the new point $S_{C_u}^{new}$ inside a ball of radius $\delta$ and center $S_{C_u}^c$, which ultimately leads to an indistinguishability between the generated traces, removing the memorization. If this is not the case, we can leave the resulting synthetic trajectory without additional modifications. Application of this procedure ensures that for each segment, we either detect memorization, which we will later remove by ensuring a $K_{\mid_{C_u}}$-anonymity within the condition, or we don't and we can use the generated trace directly without fear of memorization. 
\label{sec:framework}

\section{Experimental Evaluation}
This section presents a comprehensive evaluation of \ours from both utility and privacy perspectives. We first describe our experimental setup and datasets, then assess model's generation quality through quantitative metrics and visual inspection, and finally evaluate privacy guarantees through memorization analysis.

\subsection{Experimental Setup}
\paragraph{Datasets.} We evaluate our approach on two complementary datasets that provide controlled and realistic evaluation scenarios:

\begin{enumerate}[label=\Alph*.]
    \item \textit{Procedural dataset.} To validate our method in a controlled environment with clean data, we generated synthetic trajectories using a rule-based simulator. The simulated world contains multiple towns with two distinct mobility behaviors: (i) \textit{daily workers} (99.9\% of population) exhibit routine patterns, traveling only between home and workplace with at most two intermediate towns; (ii) \textit{travelers} (0.1\%) represent outliers with continuous movement patterns, their start an end points are systematically different. The dataset comprises 999,000 daily workers and 1,000 travelers, providing a realistic class imbalance scenario. 
    \item \textbf{GeoLife}~\cite{zheng2008understanding,zheng2009mining,zheng2010geolife}. This real-world dataset contains GPS trajectories from 182 users collected over 5 years, totaling 17,621 raw trajectories. We applied the following preprocessing: (i) spatial filtering to Beijing city boundaries where most data originates; (ii) temporal segmentation into 30-minute windows with uniform timestamps; (iii) padding to exactly 256 points per trajectory; (iv) geographic bounding to inner-city regions. This yielded 45,998 preprocessed trajectories suitable for model training. 
\end{enumerate}

\begin{table}[htb]
    \centering
    \small
    \setlength{\tabcolsep}{2pt}
    \caption{Training hyperparameters for VAE and diffusion models. Values identical across datasets unless specified.}
    \label{tab:diff_params}
    \begin{tabular}{@{}llr@{\hspace{1em}}llr@{}}
        \toprule
        \multicolumn{3}{c}{\textbf{Model Architecture}} & \multicolumn{3}{c}{\textbf{Training Configuration}} \\
        \midrule
        VAE parameters & & 3M & Optimizer & & Adam \\
        Diffusion parameters & & 11M & Learning rate (VAE) & & $10^{-3}$ \\
        Latent dimension & & 4 & Learning rate (Diff.) & & $2\times10^{-4}$ \\
        U-Net channels & & [64,128,256,512] & Batch (Proc. / Geo.) & & 256,128 / 128,64 \\
        Conditions (clusters) & & 65k & Epochs (VAE / Diff.) & & 100,200 / 300,600 \\
        & & & KL weight ($\beta$) & & 0.01 \\
        & & & CFG dropout & & 0.1 \\
        \midrule
        \multicolumn{3}{c}{\textbf{Diffusion Process}} & \multicolumn{3}{c}{\textbf{Data \& Inference}} \\
        \midrule
        Total timesteps & & 500 & Normalization & & [-1, 1] \\
        Noise schedule & & Cosine & Trajectory length & & 256 points \\
        $\beta$ range & & [0.0001, 0.02] & Samples (Proc. / Geo.) & & 999K / 46K \\
        Prediction target & & $\epsilon$ & Sampling method & & DDPM \\
        EMA decay & & 0.9999 & Sampling steps & & 500 \\
        & & & CFG scales & & [1.0, 2.5] \\
        \bottomrule
    \end{tabular}
\end{table}

\paragraph{Baseline methods.} We compare \ours against three baselines representing different generative paradigms:

\begin{enumerate}[label=(\arabic*), leftmargin=*, nosep, itemsep=2pt] 
    \item \textbf{DiffTraj}~\cite{zhu2023difftraj}: State-of-the-art conditional diffusion model for trajectory generation, conditioned on eight attributes (departure time, trip distance, trip duration, trip length, average distance and speed, start/end grid cells).
    
    \item \textbf{VAE}~\cite{kingma2013auto}: Standard variational autoencoder baseline representing classical generative modeling approaches. Provides a lower bound on achievable performance for latent variable models.
    
    \item \textbf{Gaussian noise}: Non-generative baseline that perturbs original trajectories with additive Gaussian noise ($\mu=0, \sigma=0.01$), following~\cite{zhu2023difftraj}. Represents traditional privacy-preserving perturbation methods rather than true synthesis.
\end{enumerate}

\paragraph{Implementation details.} All models have been trained on two NVIDIA GPUs (L40 24GB and RTX PRO 6000). More details on training hyperparameters are detailed in Table~\ref{tab:diff_params}. 

\subsection{Utility Evaluation}
We assess generation quality through both qualitative visual inspection and quantitative metrics spanning multiple scales of trajectory analysis.

\subsubsection{Qualitative Assessment}

Figure~\ref{fig:model_comparison} presents visual comparison of 3,000 generated trajectories from \ours and DiffTraj against real data.

On the procedural dataset (bottom row), \ours demonstrates superior fidelity in capturing the underlying mobility patterns, with generated trajectories closely following the town-to-town connectivity structure. DiffTraj exhibits higher spatial variance and occasional deviations from the road network. Both methods successfully preserve the overall spatial distribution, though \ours shows tighter adherence to infrastructure constraints.

On GeoLife (top row), both methods face greater challenges due to data sparsity and GPS noise inherent in real-world collection. \ours maintains better concentration along major transportation corridors, while DiffTraj shows more dispersed/coarse patterns. The performance gap between datasets reflects both the quantity advantage (20× more procedural data) and the cleaner distribution of synthetic data compared to noisy GPS measurements.

\begin{figure}[!htb] 
    \centering
    \begin{tabular}{ccccc}
            \begin{subfigure}[b]{0.192\linewidth}
                \centering
                \includegraphics[width=\linewidth, height=1.003\linewidth]{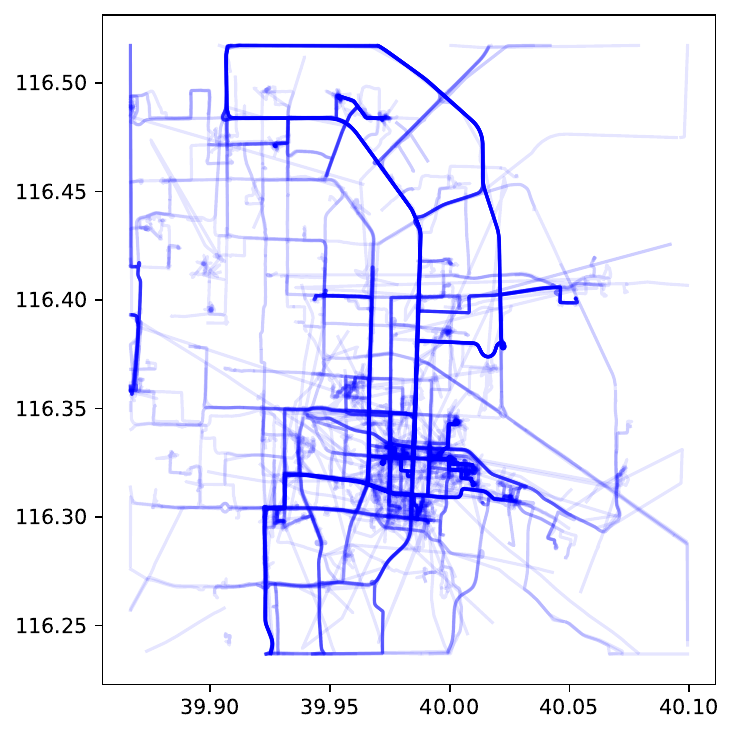}
            \end{subfigure}
            \begin{subfigure}[b]{0.192\linewidth}
                \centering
                \includegraphics[width=\linewidth, height=1.003\linewidth]{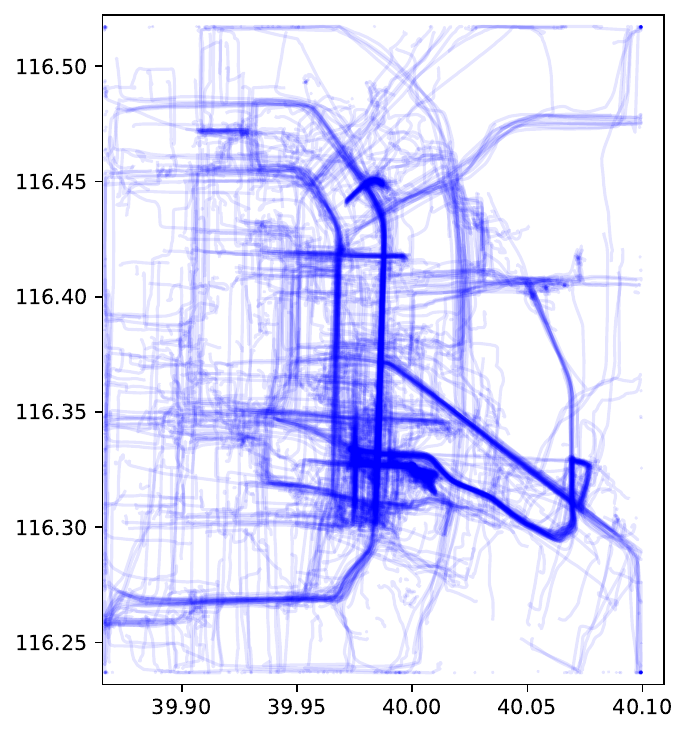}
            \end{subfigure}
            \begin{subfigure}[b]{0.192\linewidth}
                \centering
                \includegraphics[width=\linewidth]{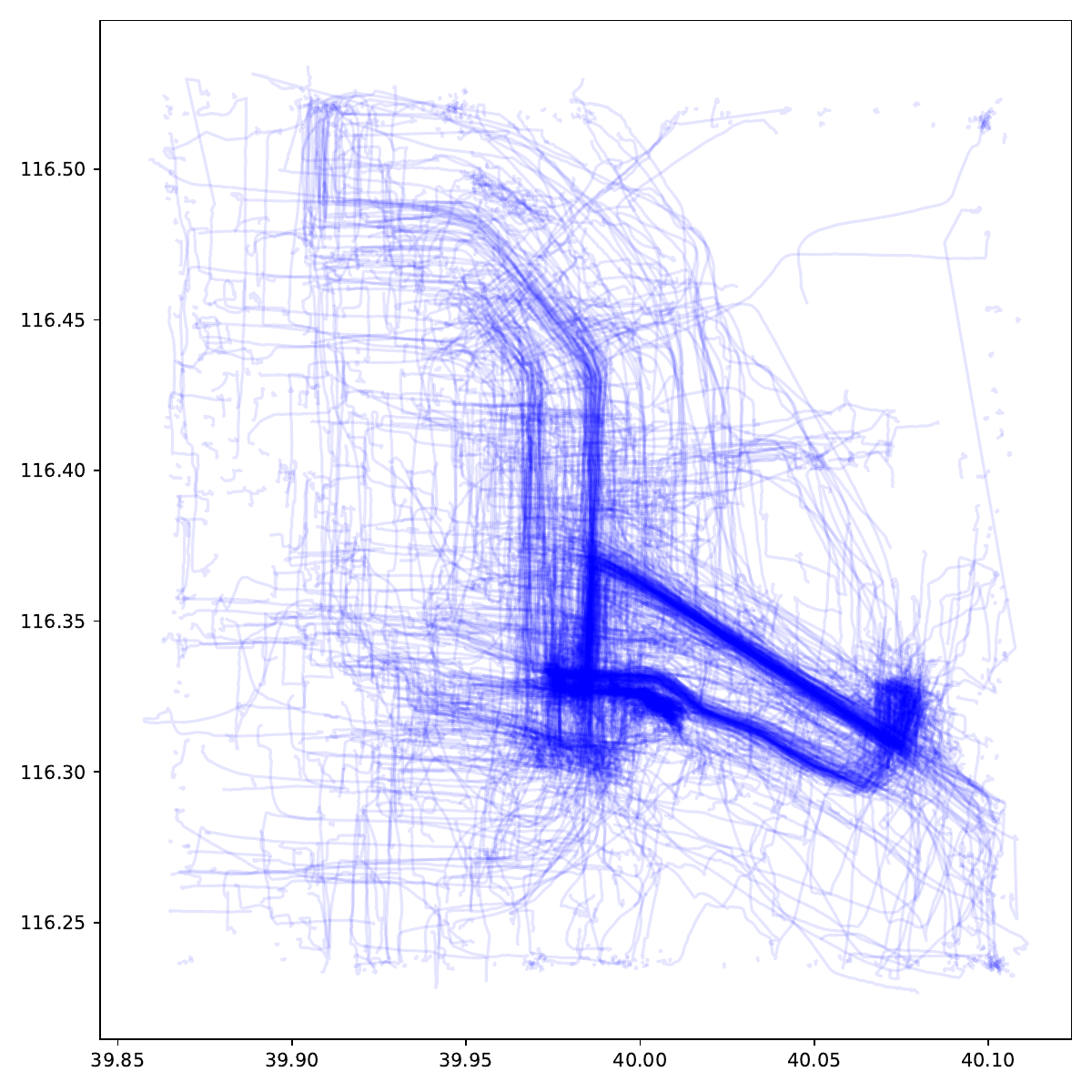}
            \end{subfigure}
            \begin{subfigure}[b]{0.192\linewidth}
                \centering
                \includegraphics[width=\linewidth, height=1.003\linewidth]{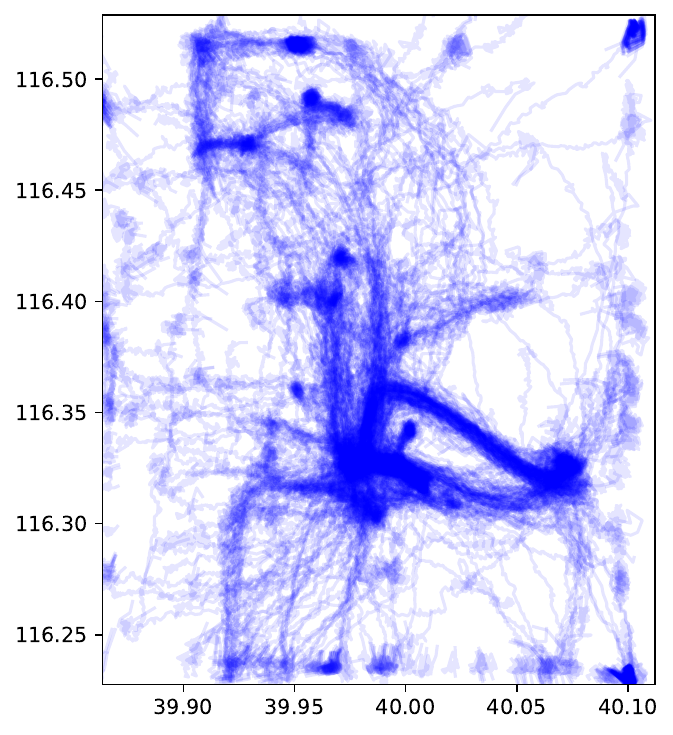}
            \end{subfigure}
            \begin{subfigure}[b]{0.192\linewidth}
                \centering
                \includegraphics[width=\linewidth, height=1.003\linewidth]{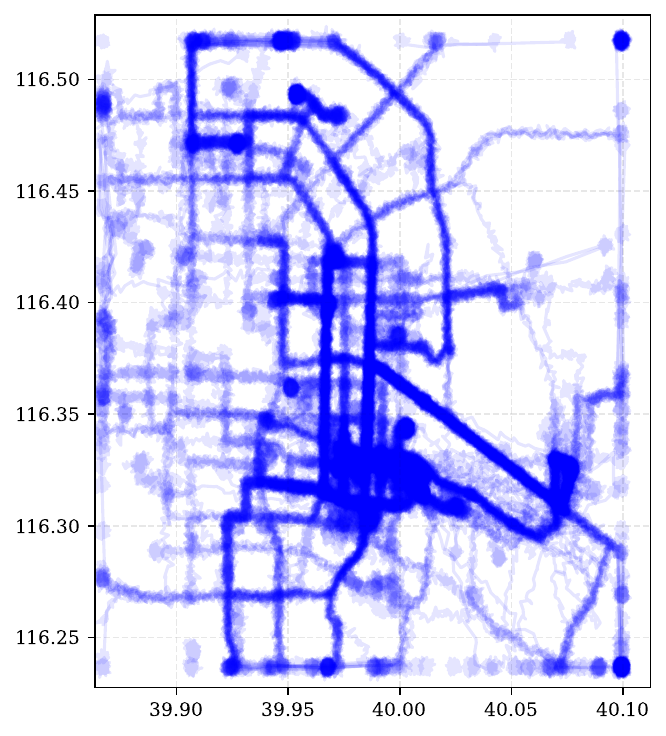}
            \end{subfigure}
            \\
        \begin{subfigure}[b]{0.192\linewidth}
            \centering
            \includegraphics[width=\linewidth]{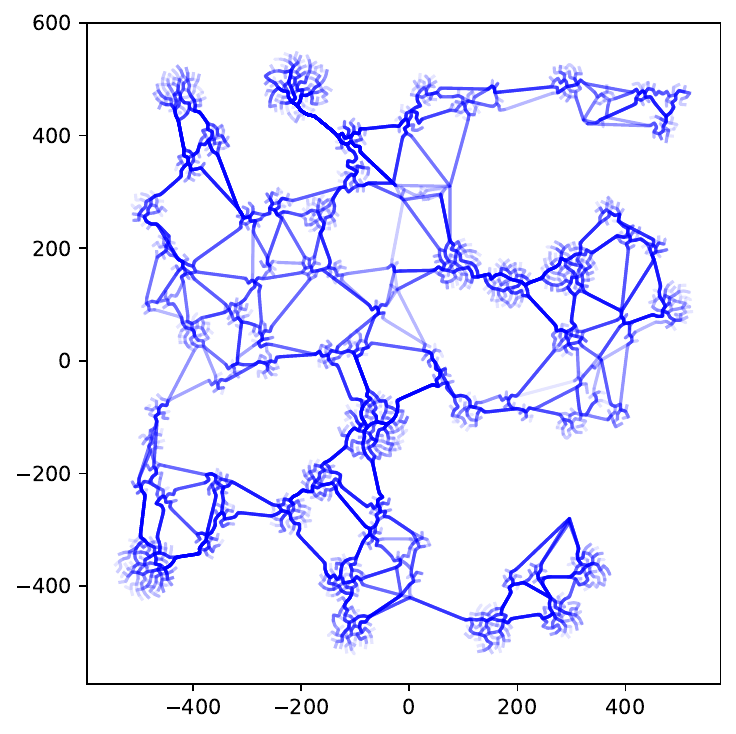}
            \caption{Real}
        \end{subfigure}
        \begin{subfigure}[b]{0.192\linewidth}
            \centering
            \includegraphics[width=\linewidth, height=1.003\linewidth]{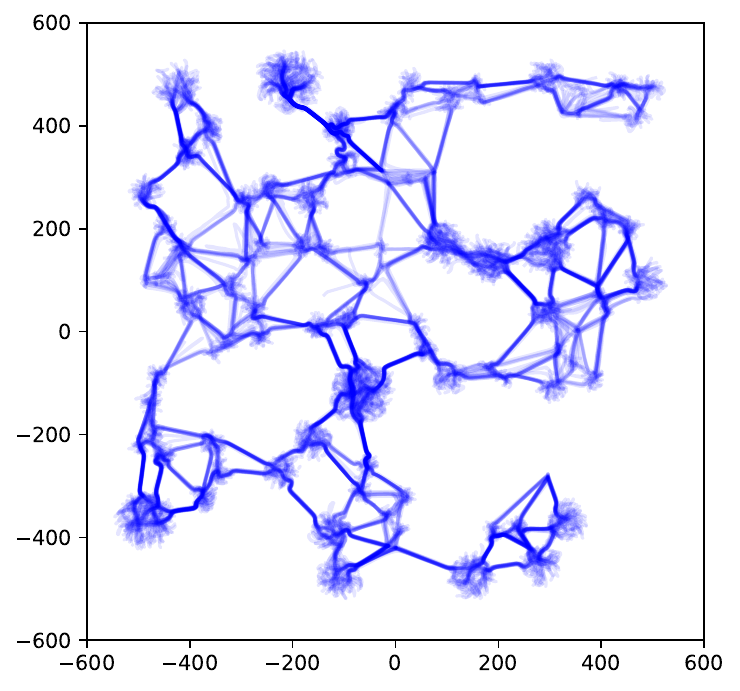}
            \caption{\ours}
        \end{subfigure}
        \begin{subfigure}[b]{0.192\linewidth}
            \centering
            \includegraphics[width=\linewidth]{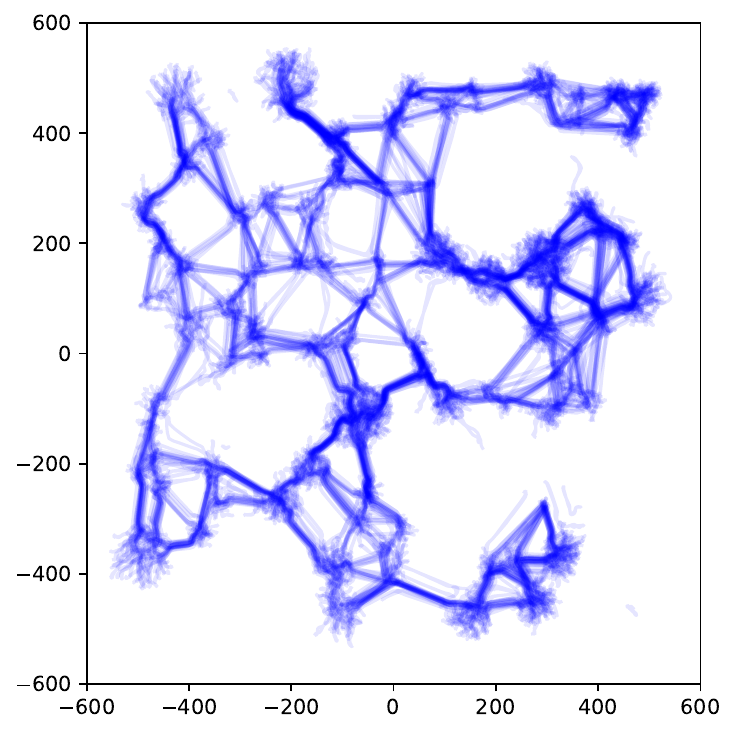}
            \caption{DiffTraj}
        \end{subfigure}
        \begin{subfigure}[b]{0.192\linewidth}
            \centering
            \includegraphics[width=\linewidth, height=1.003\linewidth]{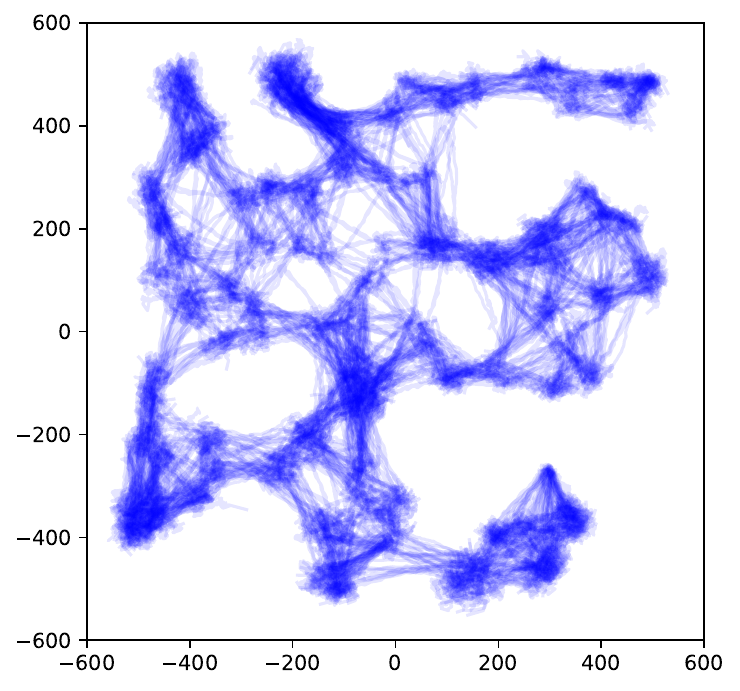}
            \caption{VAE}
        \end{subfigure}
        \begin{subfigure}[b]{0.192\linewidth}
            \centering
            \includegraphics[width=\linewidth, height=1.003\linewidth]{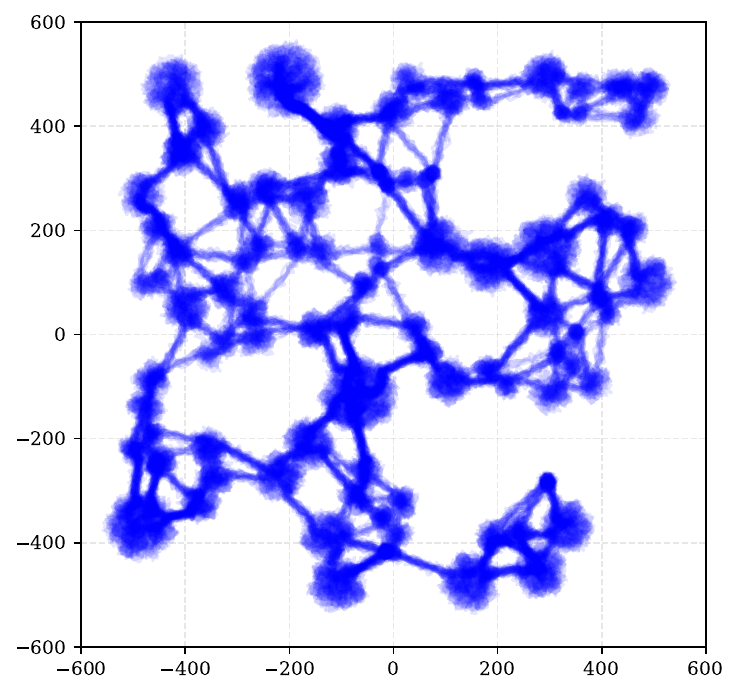}
            \caption{Gaussian}
        \end{subfigure}
    \end{tabular}
    \caption{Traces generated by the model and baseline covered by the paper. Row 1 and row 2 represent respectively illustration over GeoLife and procedural dataset.}
    \label{fig:model_comparison}
\end{figure}

\subsubsection{Quantitative Metrics}
For comprehensive utility evaluation beyond solely visual quality assessment, we selected 8 metrics classified into 2 hierarchical categories: trajectory-level and point-level metrics. This taxonomy from~\cite{cherigui2026dualperspectivesynthetictrajectory} enables evaluation of both global distributional properties and fine-grained local characteristics of generated trajectories.

\paragraph{\textbf{Trajectory-level metrics}} assess aggregate properties across entire trajectory sets. \textbf{Density Error}$^{\dagger}$ quantifies spatial distribution divergence via L1 distance over a discretized grid, ranging from 0 (identical distributions) to 2 (completely disjoint). \textbf{Pattern Score}$^{\dagger}$ measures flow pattern preservation through the F1-score of top-nn most frequently visited grid cells, capturing origin-destination dynamics (higher is better, max 1). \textbf{Average Speed} validates temporal realism via Wasserstein distance between speed distributions (m/s), where lower values indicate better preservation of mobility dynamics. \textbf{Map Reconstruction} evaluates infrastructure adherence by measuring the average Euclidean distance from synthetic trajectories to the nearest road segment in a network constructed from original data, reported in meters. Note that Density Error and Pattern Score are adopted from DiffTraj~\cite{zhu2023difftraj} for direct comparison, while the remaining six metrics provide complementary evaluation dimensions.

\begin{table}[!htb] 
    \centering
    \setlength{\tabcolsep}{0.5pt}
    \caption{Models comparison over procedural and Geolife data and across utility metrics. For each dataset, \textbf{bold} indicates best performance, \underline{underline} indicates second best. $\dagger$ Indicates metrics from DiffTraj \cite{zhu2023difftraj}. $\uparrow$/$\downarrow$ Higher/lower is better.}
    \label{tab:trajectory_metrics}
    \begin{tabular}{llcccc}
        \toprule
        \multirow{2}{*}{\textit{Trajectory-level}} & \multirow{2}{*}{Data} & \textbf{Dens Error$^{\dagger}$} & \textbf{Pattn Score$^{\dagger}$} & \textbf{Avg Speed} & \textbf{Map Recon} \\
        & & ($e^{-2}$) $\downarrow$ & ($e^{-1}$) $\uparrow$& ($e^{-4}$) $\downarrow$ & ($e^{0}$) $\downarrow$\\
        \midrule
        \multicolumn{6}{l}{\textit{Proposed Methods}} \\
        \quad Our (GS=1.0) & Proc.& $\underline{25 \pm 0.34}$ & $\underline{8.5 \pm 0.50}$ & \textbf{3.0 $\pm$ 0.0} & \textbf{2.1$\pm$0.23} \\
         & Geo.& 39 $\pm$ 1.0 & \textbf{8.4 $\pm$ 0.22} & \textbf{(5.2 $\pm$ 0.36)$e^2$} & 88 $\pm$ 1.9 \\
        \quad Our (GS=2.5) & Proc. & $27 \pm 1.10$ & \textbf{8.6$\pm$0.42} & $\underline{4.0 \pm 0.0}$ & $\underline{2.5 \pm 0.15}$ \\
         & Geo. & $\underline{38 \pm 1.1}$ & $\underline{7.8 \pm 0.56}$ & $\underline{(130 \pm 3.1)e^2}$ & 92 $\pm$ 1.0 \\
        \multicolumn{6}{l}{\textit{Baseline Methods}} \\
        \quad DiffTraj (GS=3.0) & Proc. & $53 \pm 0.15$ & $4.3 \pm 0.43$ & $5.0 \pm 0.0$ & $2.9 \pm 0.30$ \\
         & Geo. & 1.1$e^2$ $\pm$ 0.74 & 3.6 $\pm$ 0.42 & (150 $\pm$ 7.5)$e^2$ &  140 $\pm 7.5$\\
        \quad VAE & Proc. & $47 \pm 0.72$ & $5.9 \pm 0.65$ & $38 \pm 0.0$ & $2.9 \pm 0.11$ \\
        & Geo. & 49 $\pm 0.51$ & 4.5 $\pm$ 0.00 & (660 $\pm$ 0.21)$e^2$ & $\underline{86 \pm 0.31}$\\
        \quad Gaussian & Proc. & \textbf{23 $\pm$ 0.80} & $8.0 \pm 0.35$ & $260 \pm 0.0$ & $2.5 \pm 0.05$ \\
        & Geo.& \textbf{21 $\pm$ 0.22} & 6.1 $\pm$ 0.22 & (2900 $\pm$ 1.4)$e^2$ &  \textbf{76 $\pm$ 0.21}\\
        \midrule
        \multirow{2}{*}{\textit{Point-level}} & \multirow{2}{*}{Data} & \textbf{G-rank} & \textbf{Trans. Prob.} & \textbf{Loc. InPlaus.} & \textbf{Traffic Flow} \\
        && ($e^{-1}$) $\uparrow$& ($e^{-3}$) $\downarrow$& ($e^{-2}$) $\downarrow$& ($e^{-4}$) $\downarrow$\\
        \midrule
        \multicolumn{6}{l}{\textit{Proposed Methods}} \\
        \quad Our (GS=1.0) & Proc. & $\underline{7.9 \pm 0.03}$ & $1.1 \pm 0.30$ & \textbf{1.6 $\pm$ 1.2} & $9.0 \pm 0.00$ \\
        & Geo. & 6.0 $\pm$ 0.05 & 7.0 $\pm$ 1.1 & 39 $\pm$ 0.77 &  $\underline{12 \pm 0.00}$\\
        \quad Our (GS=2.5) & Proc. & $7.7 \pm 0.04$ & \textbf{0.80 $\pm$ 0.20} & $\underline{3.1 \pm 0.78}$ & $9.0 \pm 0.00$ \\
        & Geo. & $\underline{6.3 \pm 0.03}$ & $\underline{4.2 \pm 0.40}$ & 41 $\pm$ 0.69 &  \textbf{10 $\pm$ 0.00}\\
        \multicolumn{6}{l}{\textit{Baseline Methods}} \\
        \quad DiffTraj (GS=3.0) & Proc. & $6.7 \pm 0.04$ & $\underline{1.1 \pm 0.20}$ & $5.0 \pm 1.6$ & $9.0 \pm 0.00$ \\
         & Geo. & 2.6 $\pm$ 0.02 & 44 $\pm$ 3.9 & 45 $\pm$ 0.39 & 65 $\pm$ 0.00 \\
        \quad VAE & Proc. & $6.8 \pm 0.04$ & $1.8 \pm 0.00$ & $5.3 \pm 0.56$ & $12 \pm 0.00$ \\
         & Geo. & 5.8 $\pm$ 0.01 & \textbf{3.1 $\pm$ 0.20} & 39 $\pm$ 0.33 &  13 $\pm$ 0.00\\
        \quad Gaussian & Proc. & \textbf{8.0 $\pm$ 0.0} & $6.5 \pm 0.00$ & $3.3 \pm 0.27$ & $18 \pm 0.00$ \\
        & Geo. & \textbf{8.0 $\pm$ 0.04} & 5.0 $\pm$ 0.00 & \textbf{39 $\pm$ 0.12} &  15 $\pm$ 0.00\\
        \bottomrule
    \end{tabular}
\end{table}

\paragraph{\textbf{Point-level metrics}} evaluate local trajectory quality and individual point characteristics. \textbf{G-rank} assesses location popularity preservation via Kendall's $\tau_b$ correlation between frequency-based rankings of discretized locations, ranging from -1 (inverse correlation) to 1 (perfect agreement). \textbf{Transition Probabilities} evaluates first-order Markov dynamics by comparing spatial transition kernels using Wasserstein distance with spatial ground costs, where 0 indicates perfect preservation of movement patterns. \textbf{Location Implausibility} quantifies geographic plausibility as the fraction of synthetic points exceeding a distance threshold from the infrastructure network, with values in [0, 1] indicating the proportion of implausible locations. Finally, \textbf{Traffic Flow Prediction} assesses population-level mobility forecasting by applying synthetic transition matrices to predict spatial distributions, measured via time-averaged Wasserstein distance, where lower values indicate better prediction accuracy.

\paragraph{Experimental protocol.} For statistical robustness, we compute all metrics over 4 independent inference runs, each generating 10,000 trajectories with conditions randomly sampled from the training set. Results are reported as mean ± standard deviation. The infrastructure network for Map Reconstruction and Location Implausibility is constructed using grid-based discretization with 8-connectivity, filtering spatially non-adjacent transitions to eliminate artifacts.

\paragraph{Results on procedural data.} Table~\ref{tab:trajectory_metrics} presents results on synthetic procedural trajectories. Our method ( \ours) demonstrates superior performance across most metrics compared to baselines. At GS=1.0, \ours achieves the best Pattern Score and competitive Density Error, indicating strong preservation of both flow patterns and spatial distributions. Notably, point-level metrics reveal exceptional performance: G-rank correlation of $0.788$ demonstrates accurate location popularity modeling, while Transition Probabilities and Traffic Flow errors are an order of magnitude equal or lower to baselines, indicating precise capture of mobility dynamics. The low Location Implausibility ($0.016$) confirms geographic realism. The Gaussian baseline ($\mu = 0, \sigma=0.01$), despite competitive Density Error due to low level of noise, fails catastrophically on Average Speed (0.0255) and Transition Probabilities (0.0065), revealing its inability to model temporal and relational dynamics. Furthermore, one can note that, even with a small level of noise, the gaussian method destroy the realism of the trajectories in contrast to generative models. Globally, increasing CFG to 2.5 slightly degrades some metrics (e.g., Density Error increases to $0.27$) while improving others (Pattern Score: $0.8625$), suggesting a trade-off between diversity and fidelity.

\begin{figure}[!htbp]
    \centering
    \includegraphics[width=\linewidth]{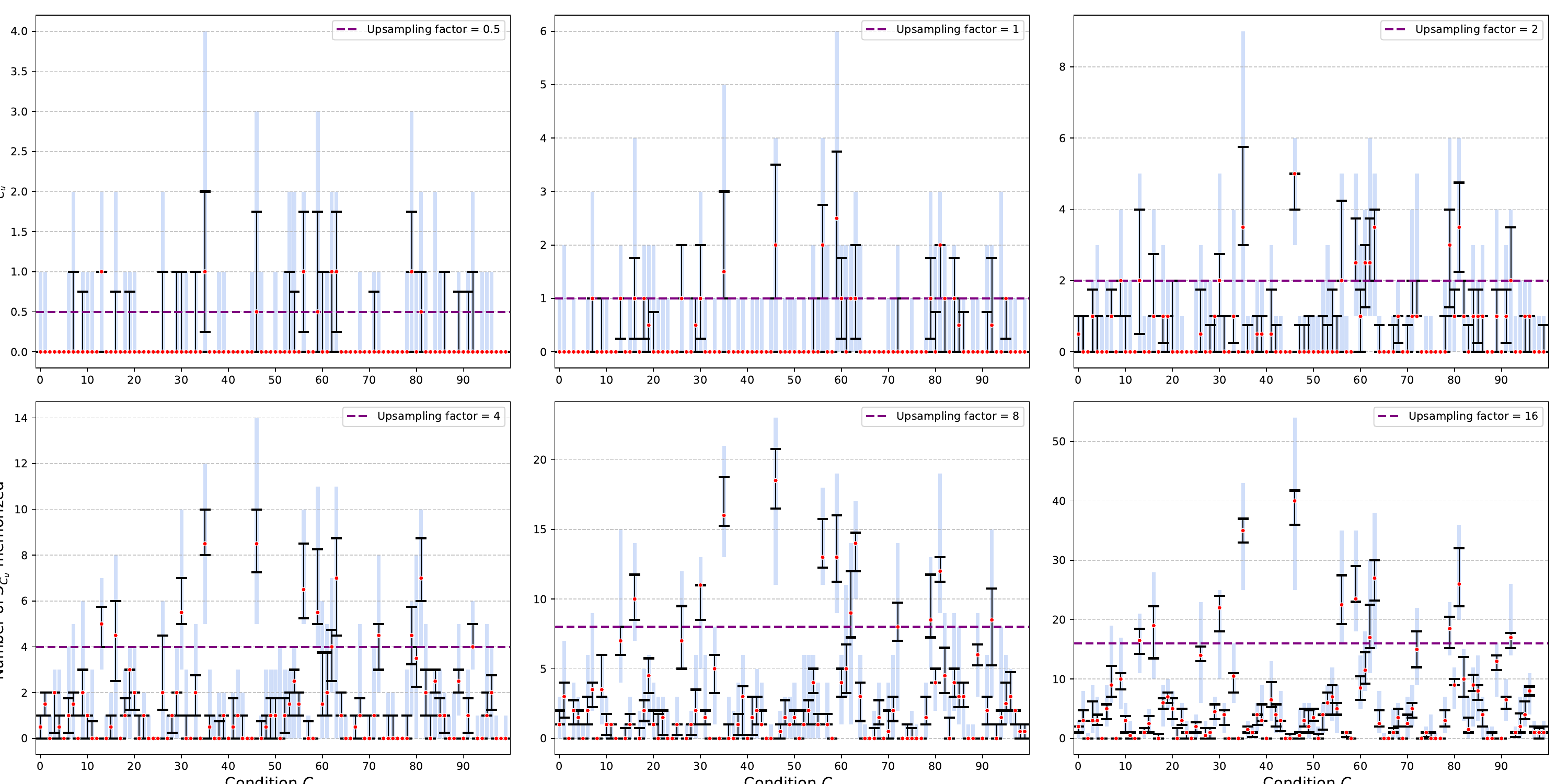}
    \caption{Number of synthetic samples detected as memorized by our framework for the procedural dataset, for various values of upsampling. Conditions leading to memorized samples above the purple threshold at least once are flagged as potentially at risks and as certain privacy risk if true for all.}
    \label{fig:memorization_results}
\end{figure}

\paragraph{Results on real-world data.} Table~\ref{tab:trajectory_metrics} shows evaluation on GeoLife GPS trajectories. \ours maintains strong performance despite the increased complexity of real-world mobility patterns. At GS=2.5, our method achieves the best Density Error and Pattern Score among the generative models, outperforming both DiffTraj and VAE. Point-level metrics further demonstrate superiority: G-rank ($0.6323$) and Transition Probabilities ($0.0042$) significantly outperform VAE, indicating better preservation of location importance and movement patterns. Interestingly, the Gaussian baseline exhibits artificially low Density Error (0.2148) but fails catastrophically on Average Speed, confirming the fact that the model doesn't take into account physics, realism. Map Reconstruction distances (87-92m) are higher than procedural data due to GeoLife's sparser sampling and GPS noise, yet \ours remains competitive with VAE (86.4m). The consistent Location Implausibility ($\sim40$) across all methods suggests this is inherent to GeoLife's data characteristics rather than model limitations.

\paragraph{Impact of classifier-free guidance (GS).} Comparing GS=1.0 and GS=2.5 reveals a consistent pattern: higher guidance improves condition alignment (better Pattern Score, G-rank) at the cost of diversity (higher Density Error, Average Speed deviation). This trade-off is more pronounced on GeoLife, where GS=2.5 increases Average Speed error by $145\%$ while improving G-rank by $6\%$, suggesting real-world data benefits from stronger conditioning to overcome inherent noise and sparsity. We noticed the same or even stronger behavior with DiffTraj. Indeed with GS <= 1, the predictions are completely noisy. Augmenting the GS to 3-5 allows to get more meaningful trajectories.

\subsection{Privacy evaluation}
In this section, we will present the choices we made to apply the test of memorization and the noise addition presented in section~\ref{subsection:privacy_framework}. We will then present the results on a random sample of 100 different conditions applied to our model, before and after application of the non-memorization guarantees. 

\subsubsection{Parameters of the privacy experiment}
To better understand the link between number of samples generated and memorization, we choose to vary the sampling factor $B$. This sampling factor is linked to the number of real trajectories associated to a given condition. Indeed, the number of generated trajectories for each condition $C_u$ is equal to $B*|K_{\mid_{C_u}}|$. We choose to vary $B$ in our experiment between $0.5$ and $16$. Because most of the value we choose are greater than 1, $B$ is also referred to as the \textbf{\textit{up}}sampling factor in the text. 

In our experimental evaluation of the privacy, we randomly selected a 100 conditions for both datasets. We limited the number of conditions to better vary the number of generated samples, with an upsampling factor going up to $16$ as previously mentioned. Throughout the experiment, we empirically choose a threshold $k=\frac{1}{2}$. In practice, we have seen that this parameter allows a human referee to make a clear distinction between two trajectories, as well as providing an upper bound on the privacy risk on condition may possess over the literature (\cite{bonnaire2025diffusion} chose a parameter of $k=\frac{1}{3}$ and remarked that a greater range of parameters (e.g. $k=\frac{1}{2}$) would not impact the results).
Starting from a theoretically grounded bound $\delta = L*\frac{d}{2}$, coming from the application of the theory to a VAE (where $L$ is the Lipschitz constant of the VAE's decoder function, and $d$ the value of the diagonal from our condition), we empirically decreased the amount of noise needed to be applied to the synthetic traces generated by a diffusion model to mitigate memorization. To validate our choice, we progressively decreased the amount of noise, and validated our results by evaluating again the memorization in the given conditions. Appendix~\ref{app:privacy_bounds} details more theory from which we derived our noise addition for the diffusion model. 

\subsubsection{Results}
Figure~\ref{fig:memorization_results} report our results over the 100 different randomly sampled condition. We conducted the upsampling 10 times for each histogram, for each value of upsampling. Boxplots are covering the range of 25\%-75\% of the number of memorized samples overs the 10 different experiments, and highlight the mean value of the number of memorized samples. For a model trained on the procedural dataset, results show that 18\% of the conditions have been flagged as at least potentially at risk. This very significant number highlight the need to a specific treatment applied to the synthetic traces generated based on those conditions. Furthermore, among the 18 conditions flagged, 8 of them are flagged as certain privacy risks, e.g. the count of synthetic samples memorized is permanently above the threshold of the upsampling factor. This is highly significant, as the procedural dataset contains almost a thousand real records. It also show that studying the privacy risks by condition is needed to better understand where in the distribution the generalization fails and lead to a memorization regime\cite{bonnaire2025diffusion,favero2025bigger}. After noise addition, done specifically on those flagged conditions, we conducted again the study of memorization. From this study, we observed that not a single condition have been flagged as potentially at risk, leading to an absence of memorization. Our procedure effectively limited the privacy risks presented by our diffusion model on the procedural dataset. We conducted a similar analysis, on the geolife dataset. Results are even more remarkable, as 51\% of the selected conditions are flagged as potentially at risk. This was however expected, as the Geolife dataset is much more sparse (it possess only 4096 possible conditions compared to the 65536 possible conditions for the procedural dataset). Similarly to the procedural dataset, the number of memorized condition decreased to 0 after conducting our procedure. 
Finally, we vary the amount of noise added to the trajectories. We performed similar experiments has the above, but varying the parameter $\delta$. We used $\delta_i = L*\frac{d}{2^i}$, and studied the amount of memorization for each $\delta_i$. We found that for $i \in [\!| 1, 7|\!]$, our method effectively reduces the risk of memorization. High $i$ however, leading to traces very similar to the barycenter, could contain traces of memorization. We then computed the utility of the generated traces, before and after the noise addition step. Interestingly, we found that for almost every metric of utility, a low amount of noise affects the utility as much as a very low noise. Indeed, in the first case the traces are very chaotic for the selected condition, and in the later they are all reduced to the barycenter of the real traces, reducing the amount of utility carried by the trajectories. However, we found that for intermediary value, such as $i=\{3,4\}$, the amount of noise is low enough for the traces to not be too chaotic, while not to high to still be diverse. The G-rank for example, increased by 17.52\% between $i=1$ and $i=4$. Similarly, the density error decreased by 39.3\% between $i=1$ and $i=3$, or increased by 113.2 \% between $i=3$ and $i=7$. 
\label{sec:experiments}

\section{Discussion}
Overall, both qualitative and quantitative results demonstrate that our model, combining VAE latent space segmentation for conditioning a diffusion model, consistently outperforms state-of-the-art methods. Beyond utility, our privacy evaluation reveals a more nuanced picture.

\paragraph{Memorization and privacy.} Diffusion models tend to naturally memorize training samples under standard training procedures, consistently with prior work~\cite{bonnaire2025diffusion,favero2025bigger} identifying training dataset size, duration, and optimizer as key factors. We expect model architecture and regularization to play a role as well. Importantly, our per-condition analysis shows that memorization is far from uniform: on the procedural dataset, 18\% of conditions were flagged as at least potentially at risk, and 8 of them as certain privacy risks — their memorized sample count permanently exceeding the upsampling threshold. This highlights that studying privacy at the condition level is essential to understand where generalization fails~\cite{bonnaire2025diffusion,favero2025bigger}. 
Targeted noise addition on flagged conditions successfully reduced this risk, validating our condition-aware mitigation strategy.

\paragraph{Limitations and future work.} Our privacy evaluation focuses on memorization, but other risks — membership inference, attribute inference attacks — remain unexplored and are left for future work. We observed how the amount of noise (and how the noise it applied) matters within our experiment. Concurrent work~\cite{mishra2026k} introduces a noise mechanism providing $k$-anonymity guarantees, whose integration on flagged conditions is a promising direction. We left for future work to explore this specific noise addition, to find the optimal noise answering the utility-privacy trade-off. Finally, our model is scoped to spatiotemporal trajectories; extending it to richer representations including categorical attributes, along with adapted privacy evaluation, is another natural next step.
\label{sec:discussion}

\section{Conclusion}
In this paper, we introduced a novel trajectory synthesis framework that leverages VAE latent space segmentation to condition a diffusion model, providing targeted guarantees against memorization. Our approach gives practitioners a principled way to escape the opacity of the generalization-versus-memorization regime, which is otherwise difficult to diagnose under large training sets. We presented an extensive set of experiments evaluating our model on both utility and privacy, benchmarked against state-of-the-art generative models, and demonstrated consistent improvements across the board.
\label{sec:conclusion}

\begin{credits}

\subsubsection{\discintname}
The authors have no competing interests to declare that are relevant to the content of this article. 
\end{credits}

\bibliographystyle{splncs04}
\bibliography{bibliography}

\appendix

\section{Bound over the noise added to the trajectories}
\label{app:privacy_bounds}

The intuition behind our choice of bound is to encompass the whole range of possible synthetic trajectories spawned by a given fixed condition $C_u$. The bound needs to be tight enough to preserve the overall utility of the generated traces, while ensuring the memorization goes down to 0. 

To fix an upper bound for $\delta$, we need the following theorem:
\begin{theorem}
    \label{thm:k_anonymity}
    Let $\mathcal{Z}$ the latent space of a VAE, and $f:\mathcal{Z}\rightarrow \mathbb{R}^n$ the decoder function (made of RelUs) of this VAE. Then, if we define $\mathcal{R}$ a hyper rectangle of diagonal $d_{\mathcal{R}}$ in $\mathcal{Z}$, the following holds: $\forall z\in\mathcal{R},~f(z) \in \mathcal{B}(f(z_c), \delta_{\mid_{\mathcal{R}}})$ with $\delta_{\mid_{\mathcal{R}}} = L_{\mathcal{R}}*\frac{d_{\mathcal{R}}}{2}$ and $L_{\mathcal{R}} = \text{sup}_{(x,y) \in \mathcal{R}} \frac{||f(x) - f(y)||}{||x-y||}$.
\end{theorem}
\begin{proof}
    Consider a VAE, made of linear layers and RelUs activation. The linear layers are directly $l$-Lipschitz, and the RelUs activation are $1$-Lipschitz because $RelU(x) = max(0,x)$ therefore, by dichotomy: 
    \begin{align*}
            x\leq 0~\text{and}~y \leq 0~\text{then}&|RelU(x) - RelU(y)| = |0-0| &\leq |x-y| \\
            x\geq 0~\text{and}~y \geq 0~\text{then}&|RelU(x) - RelU(y)| = |x-y| &\leq |x-y| \\
            x > 0~\text{and}~y < 0~\text{then}~&|RelU(x) - RelU(y)| = |x-0| = x &\leq |x-y|\\ 
            &\text{Because if}~y~\text{is negative,}~x-y > x. 
    \end{align*}
    And the same holds by element-wise RelU activation function. Therefore, since every layer of the decoder is a Lipschitz function, by composition of Lipschitz functions the decoder of the VAE is L-Lipschitz. Let $f$ the function of the decoder. We want to prove that with $\mathcal{R}$ a hyper rectangle of diagonal $d_{\mathcal{R}}$ in $\mathcal{Z}$ the latent space of the VAE, and $z_c$ the center of $\mathcal{R}$, $\forall z\in\mathcal{R},~|f(z) - f(z_c)| \leq L*\frac{d_{\mathcal{R}}}{2}$.
    Let $z \in \mathcal{R}$, by L-Lipschitz property: $|f(z) - f(z_c)| \leq L*|z-z_c|$. The maximum value of $|z-z_c|$ in $\mathcal{R}$ is obtained when $z$ is on the edge of $\mathcal{R}$. The maximal distance between the center of a hyper-rectangle and point on the edge is by definition $\frac{d_{\mathcal{R}}}{2}$. Thus the theorem holds. 
\end{proof}
The exact bound for $\delta$, if applied to a VAE, then becomes the bound obtain by Theorem~\ref{thm:k_anonymity}: $\delta = L*\frac{d}{2}$. 
\begin{algorithm}[htbp!]
  \caption{Power Iteration Method}
  \label{alg:power_iteration}
  \begin{algorithmic}
    \STATE {\bfseries Inputs:} $g$, $N$
    \STATE Sample a random vector $V_0$ such that $||V_0||_2 = 1$.
    \REPEAT
    \STATE Compute $V_{k+1} = (J^TJ)V_k$
    \STATE Normalize $V_{k+1}$
    \UNTIL{$N$ iterations}
    \STATE {\bfseries Outputs:} $\sigma_{max} = ||JV_N||_2$
  \end{algorithmic}
\end{algorithm}
Now, in practice we need to adapt the theory to our use-case of a conditional diffusion model, where the condition are hyperrectangle from the latent space of a VAE. We empirically choose a CFG high enough (2.5) to ensure that if a condition is made using $\mathcal{R}$, the generated traces can be mapped back into $\mathcal{R}$ by the VAE. 
Then, we needed to compute the constant $L$ for our conditional diffusion model. Computing $L$ directly is impossible given the dimension of the problem. Instead, we computed $L$ locally, by doing the approximation that the diffusion process $g$ is differentiable for every gaussian sample $x$ on a given condition $C_u$. It means we can write:
\begin{equation}
\label{eq:diff}
    \begin{cases}
        &g_{C_u}(x+h) = g_{C_u}(x) + J(x)h + o(h) \\
        e.g.& g_{C_u}(x+h) - g_{C_u}(x) \sim J(x)h
    \end{cases}
\end{equation}
with $J$ the Jacobian of $g$ in $x$. Since $g$ is L-Lipschitz, we can also write $||g(x) - g(y)||_2 \leq L||x-y||_2$, then with $h = y-x$, we can use it and equation~\ref{eq:diff} to write: $\frac{||J(x)h||_2}{||h||_2} \leq L$.
Which lead to $L = sup_{h \neq 0} \frac{||J(x)h||_2}{||h||_2} = ||J(x)||_2 = \sigma_{max}(J(x))$ with $\sigma_{max}$ the greatest singular value of $J(x)$. We still have problem with the great number of dimension, indeed the computation of the Jacobian is, in practice, not possible (It may contain billions of element for a Unet). However we only need the value of its greatest singular value. To achieve our goal, we give an approximation to its value by the power iteration method. Algorithm~\ref{alg:power_iteration} contains a description of the method to find the greatest singular value of the Jacobian without having to compute the full Jacobian matrix. After application of Algorithm~\ref{alg:power_iteration}, we provide a value, for a given $x$, of $L_{C_u}(x)$. We finally apply the Monte-Carlo method to find the global value $L_{C_u}$. In our paper, we do so by applying the above methodology for 300 different values of $x$ and $t$, for a given condition $C_u$.

\end{document}